\documentclass{llncs}
\usepackage{xcolor}
\usepackage{soul}
\usepackage[utf8]{inputenc}
\usepackage[small]{caption}
\usepackage{comment}

\usepackage{amssymb,amsmath}

\usepackage{longtable}

\usepackage{graphicx} 
\usepackage{graphics, color}
\usepackage{algorithm}
\usepackage{algpseudocode}
\usepackage{longtable}
\usepackage{tabularx}

\usepackage{url}
\newfloat{algorithm}{t}{lop}
\usepackage{enumitem}
\usepackage{xcolor}
\newcommand{\weight}[1]{\ensuremath{W \! \left( #1 \right)}}

\newcommand{\NP}{\ensuremath{\mathsf{NP}}}

\newcommand{\MaxSAT}{\ensuremath{\mathsf{MaxSAT}}}

\newcommand{\MaxHS}{\ensuremath{\mathsf{MaxHS}}}

\newtheorem{construction}[theorem]{Construction}

\newcommand{\Rule}{\ensuremath{{\mathcal{R}}}}
\newcommand{\cE}{\mathcal{E}}
\algnotext{EndFor}
\algnotext{EndIf}
\algnotext{EndWhile}
\algblockdefx{Do}{EndDo}{\textbf{do}\;}{}
\algnotext{EndDo}
\usepackage{booktabs}
\let\oldbibitem\bibitem
\def\bibitem{\vfill\oldbibitem}
\DeclareMathAlphabet\mathbfcal{OMS}{cmsy}{b}{n}
\newcommand{\ra}[1]{\renewcommand{\arraystretch}{#1}}
\newcommand{\SLIC}{\ensuremath{\mathcal{MLIC}}}
	
	\title{MLIC: A MaxSAT-Based framework for learning interpretable classification rules\thanks{The names of authors are sorted alphabetically by last name and the order does not reflect contribution}}
	
	\author{Dmitry Malioutov\inst{1} \and
		Kuldeep S. Meel\inst{2}
	}
	\institute{
	 T.J. Watson IBM Research center\\
	 \email{dmal@alum.mit.edu}\\
	 \and
 	School of Computing, National University of Singapore\\
 	\email{meel@comp.nus.edu.sg}\\
 }
	
\begin{document}
	\maketitle	
\begin{abstract}
The wide adoption of machine learning approaches in the industry, 
government, medicine and science has renewed the 
interest in interpretable machine learning: many decisions
are too important to be delegated to black-box techniques 
such as deep neural networks or kernel SVMs. Historically, problems
of learning interpretable classifiers, including classification rules 
or decision trees, have been approached by greedy heuristic methods 
as essentially all the exact optimization formulations are NP-hard.  
Our primary contribution is a MaxSAT-based framework, called {\SLIC},
which allows principled search for interpretable classification 
rules expressible in propositional logic. Our approach benefits 
from the revolutionary advances in the constraint satisfaction 
community to solve large-scale instances of such problems. In 
experimental evaluations over a collection of benchmarks arising 
from practical scenarios we demonstrate its effectiveness: we show 
that the formulation can solve large classification problems 
with tens or hundreds of thousands of examples and thousands of 
features, and to provide a tunable balance of accuracy vs. 
interpretability.  Furthermore, we show that in many problems 
interpretability can be obtained at only a minor cost in accuracy.\\

 The primary objective of the paper is to show that recent advances in 
the MaxSAT literature make it realistic to find optimal (or very high quality
near-optimal) solutions to large-scale classification problems. We also
hope to encourage researchers in both interpretable classification 
and in the constraint programming community to take it further and  
develop richer formulations, and bespoke solvers attuned to 
the problem of interpretable ML. 
\end{abstract}



\section{Introduction}

The last decade has witnessed an unprecedented adoption of machine 
learning techniques to make sense of available data and make predictions 
to support decision making for a wide 
variety of applications ranging from health-care analytics to 
customer churn predictions, movie recommendations and macro-economic policy. The focus in the machine learning literature has been on 
increasingly sophisticated systems with the paramount goal of 
improving the accuracy of their predictions at the cost of making 
such systems essentially black-box. While in certain 
tasks such as ad predictions, accuracy is the main 
objective, in other domains, e.g., in legal, medical, and 
government, it is essential that the human decision makers who 
may not have been trained in machine learning can interpret and 
validate the predictions \cite{freitas2014comprehensible,varshney_dichotomy}.

The most popular interpretable techniques that tend to be adopted and trusted
by decision makers include classification rules, decision trees, and decision lists \cite{Cohen1995,quinlan2014,breiman1984classification,Rivest1987}. In particular, decision rules with a small number of Boolean clauses tend to be the most interpretable. Such models can be used both to learn interpretable models from 
the start, and also as proxies that provide post-hoc explanations to pre-trained 
black-box models \cite{craven1996_black_box_trees,andrews1995survey_rule_induction}. 

On the theoretical front, the problem of rule learning was shown to be computationally intractable
\cite{Valiant1985}. Consequently, the earliest practical efforts such as decision list and decision tree approaches relied on a combination of heuristically chosen optimization objectives and greedy algorithmic techniques, and the size of the rule was controlled by either early stopping or ad-hoc rule pruning. Only recently there have been some formulations that attempt to balance the 
accuracy and the size of the rule in a principled optimization objective either through 
combinatorial optimization, linear programming (LP) relaxations, submodular optimization, or Bayesian methods \cite{BertsimasCR2012,marchand2002SetCoverMachine,malioutov2013exact}\cite{boros2000-LogicalAnalysisOfData,wang2015or} as we review in Section \ref{sec:relatedWork}.

Motivated by the significant progress in the development of combinatorial solvers (in particular, {\MaxSAT}), we ask: {\em can we design a combinatorial framework to efficiently construct interpretable classification rules that takes advantage of these recent advances?} The primary contribution of this paper is to present a combinatorial framework that enables 
a precise control of accuracy vs. interpretability, and to verify that the computational advances in the {\MaxSAT}community can make it practical to solve large-scale classification problems. 

In particular, this paper makes following contributions:
\begin{enumerate}
	\item A {\MaxSAT}-based framework, {\SLIC}, that provably trades off accuracy vs.     interpretability of the rules
	\item A prototype implementation of {\SLIC} based on {\MaxSAT} that is capable of finding optimal (or high-quality near-optimal) classification rules from modern large-scale data-sets
	\item We show that in many classification problems interpretability can be achieved at only a minor loss of accuracy, and furthermore, {\SLIC}, which specifically looks for interpretable rules, can learn from much fewer samples than black-box ML techniques.
\end{enumerate}

Furthermore, we hope to share our excitement with applications of constraint programming/{\MaxSAT} in Machine Learning, and to encourage researchers in both interpretable classification and in the CSP/SAT communities to consider this topic further: both in developing new SAT-based formulations for interpretable ML, and in designing bespoke solvers attuned to the problem of interpretable ML.  

The rest of the paper is organized as follows: We discuss notations and preliminaries in Section~\ref{sec:prelims}. 
We then present {\SLIC}, which is the primary 
contribution of this paper, in Section~\ref{sec:slic} and follow up with experimental setup and results over 
a large set of benchmarks in Section~\ref{sec:evaluation}. We then discuss related work in Section~\ref{sec:relatedWork} and finally conclude in Section~\ref{sec:conclusion}.


\section{Preliminaries}\label{sec:prelims}
 We use capital  boldface letters such as $\mathbf{X}$ to denote matrices while lower boldface letters $\mathbf{y}$ are reserved for vectors/sets. For a matrix $\mathbf{X}$, $\mathbf{X}_i$  represents i-th row of $\mathbf{X}$ while for a vector/set $\mathbf{y}$, $y_i$ represents i-th element of $\mathbf{y}$. 
 
Let $F$ be a Boolean formula and  $\mathbf{b} = \{b_1,b_2,\cdots b_n \}$ be the set of variables appearing in $F$. A literal is a variable ($b_i$) or its complement($\neg b_i$).  A \emph{satisfying assignment} or a
\emph{witness} of $F$ is an assignment of variables in $\mathbf{b}$ that makes
$F$ evaluate to \emph{true}.  If $\sigma$ is an assignment of variables and $b_i \in \mathbf{b}$, we use $\sigma(b_i)$ to denote the value assigned to
$b_i$ in $\sigma$. $F$ is in Conjunctive Normal Form (CNF) if $F := C_1 \wedge C_2 \cdots C_m$, where each clause $C_i$ is represented as disjunction of literals. We use $|C_i|$ to denote the number of literals in $C_i$.   For two vectors $\mathbf{u}$ and $\mathbf{v}$ over propositional variable/constants, we define $\mathbf{u} \vee \mathbf{v} = \bigvee_{i} (u_{i} \wedge v_{i})$, where $u_{i}$ and $v_{i}$ denote variables/constants at $i$-th index of $\mathbf{u}$ and $\mathbf{v}$ respectively. In this context, note that the operation $\wedge$ between a variable and a constant follows standard interpretation, i.e. $0 \wedge b = 0$ and $1 \wedge b = b$ .

We consider standard binary classification, where we are given a collection of 
training samples $\{ \mathbf{X}_i, y_i \}$ where each vector $\mathbf{X}_i \in \mathcal{X}$ contains valuation of the features $\mathbf{x} = \{x^1, x^2, \cdots x^m\}$ for sample $i$, and $y_i \in \{0,1\}$ is the binary label for sample $i$.
A classifier {\Rule} is a mapping that takes in a feature vector $\mathbf{x}$ and return a class $y$, i.e. $y = \Rule(\mathbf{x})$. The goal is not 
only to design $\Rule$ to approximate our training set, but also to generalize to unseen samples arising from the same distribution. 
In this work, we restrict $\mathbf{x}$ and $y$ to be Boolean\footnote{We discuss in Section~\ref{sec:slic} that such a restriction can be achieved without loss of generality} and focus on classifiers that can be expressed compactly in Conjunctive Normal Form (CNF). We use $C_i$ to denote the $i$th clause of $\Rule$. Furthermore, we use $|\Rule|$ to denote the sum of the counts of literals in all the clauses, i.e. $|\Rule| = \Sigma_i |C_i|$.


In this work, we focus on weighted variant of CNF wherein a weight function is defined over clauses. For a clause $C_i$ and 
weight function {\weight{\cdot}}, we use {\weight{C_i}} to denote the weight of clause $C_i$. We say that a clause $C_i$ is hard if $\weight{C_i} = -\infty$, otherwise $C_i$ is called as soft clause.  To avoid notational clutter, we overload {\weight{\cdot}} to denote the weight of an assignment or clause, depending on the context. We define weight of an assignment $\sigma$ as the sum of weight of clauses that $\sigma$ does not satisfy. Formally, $\weight{\sigma} = \Sigma_{i | \sigma \not\models C_i} \weight{C_i}$.

Given $F$ and weight function $\weight{\cdot}$, the problem of {\MaxSAT} is to find an assignment $\sigma^*$ that has the maximum weight, i.e.  $\sigma^* = \MaxSAT(F,W)$ if $\forall \sigma \neq \sigma^*, \weight{\sigma*} \geq \weight{\sigma}$. Our formulation will have negative clause weights, hence \MaxSAT corresponds to satisfying as many clauses as possible, and picking the weakest clauses among the unsatisfied ones. Note that the above formulation is different from the typical definition of {\MaxSAT} but the difference is only syntactic. Borrowing terminology of community focused on developing {\MaxSAT} solvers, we are solving a partial weighted {\MaxSAT} instance wherein we mark all the clauses with $-\infty$ weight as hard and negate weight of all the other clauses and ask for a solution that optimizes the partial weighted {\MaxSAT} formula.  The knowledge of inner working of {\MaxSAT} solvers and encoding of our representation into weighted {\MaxSAT} is not required for this paper and we defer the details to release of source code post-publication.



%


\section{{\SLIC}: MaxSAT-based Learning of Interpretable Classifiers}
\label{sec:slic}
We now discuss the primary technical contribution of this paper, {\SLIC}: {\MaxSAT}-based Learning of Interpretable Classifiers. We first describe a metric for interpretability of CNF rules. Since our formulation employs binary features, we discuss how non-binary features such as categorical and continuous features can be represented as binary features. We then move on to formulate the problem of learning interpretable classification rules as a {\MaxSAT} query and provide a proof of its theoretical soundness regarding controlling sparsity of the rules. As discussed in Section~\ref{sec:relatedWork}, prior work does not provide a sound procedure for controlling sparsity and accuracy. We then discuss the representational power of our CNF framework -- in particular, we demonstrate that the proposed framework generalizes to handle complex objective function and rules in forms other than CNF. 

\subsection{Balancing Accuracy and Intrepretability}

While in general interpretability may be hard to define precisely, in the context of decision rules, an effective proxy is merely the count of clauses or literals used in the rule. Rules involving few clauses with few literals are natural for humans to evaluate and understand, while complex rules involving hundreds of clauses will not be interpretable even if
the individual clauses are. In addition to interpretability, such sparsity also controls model complexity and gives a handle of the generalization error.\footnote{The framework proposed in this paper allows generalization to other forms of rules, as we discuss in Section~\ref{sec:generalization}.}

First, suppose that there exists a rule $\Rule$ that perfectly classifies all the examples, i.e. 
$\forall{i}, y_i = \Rule(\mathbf{X}_i)$. Among all possible functions that satisfy this we would like to
find the most interpretable (sparse) one:
\begin{align*}
\min_{\Rule}   |{\Rule}|  ~~ \mbox{ such that } \Rule(\mathbf{X}_i) = y_i, ~~\forall{i}
\end{align*}

Since most ML datasets do not allow perfect classification, we introduce a penalty on 
classification errors. We balance the two terms by a parameter $\lambda$, where large $\lambda$ 
gives more accurate but more complex rules, and smaller $\lambda$ gives smaller rules at the 
cost of reduced accuracy. Let $\cE_{\Rule}$ be the set of examples on which our classifier 
$\Rule$ makes an error, then our objective is\footnote{Cost-sensitive classification is defined analogously 
	by allowing a separate parameter for false positives and false negatives.}:

\begin{equation}\label{eq:obj-sparsity}
\min_{\Rule} |{\Rule}| +  \lambda | \cE_{\Rule} |  ~~ \mbox{ such that } {\Rule}(\mathbf{X}_i) = y_i, ~~\forall{i \notin \cE_{\Rule}}
\end{equation}


\subsection{Discretization of Features}


In our {\MaxSAT}-based formulation, we focus on learning rules based on Boolean variables.  We do also allow categorical and continuous features for our classifier,  which are pre-processed before being presented to the \MaxSAT-formulation.   To handle categorical features one may use the common `one-hot' encoding, where a Boolean vector variable is introduced with the cardinality equal to the number of categories.  For example a categorical feature
with values {'red', 'green', 'blue'} would get converted to three binary variables,  which take values 100, 010, and 001
for the three categorical values.    [[[ If you think this sentence is obvious -- please drop ]]]

For continuous features, we introduce discretization, by comparing feature values to a collection of thresholds. The thresholds may be chosen for example based on quantiles of their distribution, or alternatively, on uniform partition of the range of feature values. Specifically, for a continuous feature $x^{c}$ we consider a number of thresholds $\{ \tau_k \}$ and define two separate  Boolean features  $I[x^{c} \ge \tau_k ]$ and $I[x^{c} < \tau_k ]$  for each $\tau_k$.   The number of thresholds may vary by feature.  Thus, each continuous feature is represented using a collection of $2 q$ Boolean features, where $q$ is the number of thresholds.

In principle, one could use all the values occurring in the data as thresholds, and this would be equivalent to the original continuous features.  In practice, however, such granularity is typically not necessary, and a handful of thresholds could be used, e.g., age-groups for each $5$ years to discretize a continuous age variable. This typically leads to only a very minor (if any) loss in accuracy, and in fact improves the presentations and understanding of the rules to human users.
In our experiments, we used  10 thresholds based on the quantiles of the feature distribution (10-th, 20-th, ... 100-th percentile), unless the number of unique values of the feature was less than 10, in which case we kept all of them.

We note that we could easily define arbitrary other Boolean functions of continuous or categorical variables within our framework.  For example,  categorical variables  with many possible values (e.g.  states or countries) may be grouped into  more interpretable coarser units ( regions  or continents).  Such groupings are application specific and wpuld typically require relevant domain knowledge. They could perhaps be learned from data, but this is outside the scope of the current paper.

\subsection{Transformation to Max-SAT query}
We now describe our Max-SAT formulation for learning interpretable rules. {\SLIC} takes in four inputs: (i) a (0,1)-matrix ${\mathbf{X}}$ of dimension $n\times m$ describing values of all $m$ features for $n$ samples with $\mathbf{X}_i$ corresponding to feature vector $\mathbf{x} = \{x^1, x^2, \cdots x^m\}$ for sample $i$, (ii)  (0,1)-vector $\mathbf{y}$ containing class labels $y_i$ for sample $i$, (iii) k, the desired number of clauses in CNF rule, (iv) the regularization parameter $\lambda$.  Consequently, {\SLIC} constructs a {\MaxSAT} query and invokes a $\MaxSAT$ solver to compute the underlying rule {\Rule} as we now describe.

The key idea of {\SLIC} is to define a {\MaxSAT} query over $k \times m$ propositional variables, denoted by $\{ b^{1}_{1}, b^{2}_{1}, \cdots b^{m}_{1} \cdots b^{m}_{k}\}$, such that  every truth assignment $\sigma$ defines a k-clause CNF rule $\Rule$, where feature $x^{j}$ appears in clause $\Rule_{i}$ if $\sigma(b^{j}_{i}) = 1$. Corresponding to every sample $i$, we introduce a noise variable $\eta_i$ that is employed to distinguish whether the labeling for sample $i$ should be considered as noise or not. Let $\mathbf{B}_{i} = \{b_{i}^{j} \mid j \in [m] \}$. 

The Max-SAT query constructed by {\SLIC} consists of the following three sets of constraints:
\begin{enumerate}
	\item $N_i:= (\neg \eta_i ); \qquad \qquad \weight{N_i} = -\lambda$
	\item $V_i^j:= (\neg b_i^j); \qquad \qquad  \weight{V_i^j} = -1$
	\item $D_i:= (\neg \eta_i \rightarrow ( y_i \leftrightarrow \bigwedge_{l=1}^{k} ({\mathbf{X}_{i}} \vee {\mathbf{B}_{l}})));    \weight{D_i} = -\infty$
\end{enumerate}
Please refer to Section~\ref{sec:prelims} for the interpretation of  $({\mathbf{X}_{i}} \vee {\mathbf{B}_{j}})$.
Finally, the set of constraints $Q^k$ constructed by {\SLIC} is defined as follows:
\begin{equation}
Q^k := \bigwedge_{i=1}^{n} N_i \wedge \bigwedge_{i=1, j = 1}^{i=k,j=m} V_i^j \wedge \bigwedge_{i=1}^{n} D_i
\end{equation}

Note that the elements of $\mathbf{X}_{i}$ and $y_i$ are not variables but constants whose values (0 or 1) are provided as inputs. Therefore, the set of variables for $Q^k$ is  $\{\eta_1, \eta_2, \cdots, \eta_n, $ $ b^{1}_{1}, b^{2}_{1}, \cdots b^{m}_{1} \cdots b^{m}_{k} \}$. We now explain the intuition behind the design of $Q^k$. 

We assign a weight of $-\lambda$ to every $N_i$ as we would like to satisfy as many $N_i$, i.e. falsify as many $\eta_i$ as possible. Similarly, we assign a weight of $-1$ to every clause $V_i^j$ as we are, again, interested in sparse solutions (i.e., ideally, we would prefer as many $V_i^j$ to be satisfied as possible).  Every clause $D_i$ can be read as follows: if $\eta_i$ is assigned to false, i.e. sample $i$ is not considered as noise, then $\mathbf{y}_i = \Rule$. As noted in Section~\ref{sec:prelims}, equivalent representation of the $\weight{\cdot}$, as described above, for {\MaxSAT} solvers involves usage of hard clauses.  

Next, we extract $\Rule$ from the solution of $Q^k$ as follows. 
\begin{construction}
	Let $\sigma^* = \MaxSAT(Q^k,W)$, then $x^j \in {\Rule}_i$ iff $\sigma^*(b_{i}^{j}) = 1$.
\end{construction}

Before proceeding further, it is important to discuss CNF encodings for the above sets of constraints. The constraints arising from $N_i$ and $V_i$ are unit clauses and do not require further processing. Furthermore, note that $\mathbf{y}_i$ is already known and is a constant. 
Therefore, when $\mathbf{y}_i$ is 1, the constraint $D_i$ can be directly 
encoded as CNF by using equivalence of $(a \rightarrow b) \equiv (\neg a \vee b)$. Finally, when $\mathbf{y}_i$ is 0, we use Tseitin encoding wherein 
we introduce an auxiliary variable $z_i^j$ corresponding to each clause $({\mathbf{X}_{i}} \vee \mathbf{B}_{j})$. Formally, we replace $D_i:= (\neg \eta_i \rightarrow (\bigvee_{j=1}^{k} \neg ({\mathbf{X}_{i}} \vee \mathbf{B}_{j})))$ with $\bigwedge_{j=0}^{k} D_{i}^{j}$ where $D_{i}^{0} := (\neg \eta_i \rightarrow  \bigvee_{j} z_{i}^{j}))$, and  $D_{i}^{j} := (z_{i}^{j} \rightarrow \neg ({\mathbf{X}_{i}} \vee \mathbf{B}_{j})$. 
Furthermore, $\weight{D_i^j} = -\infty$. 
The following lemma establishes the theoretical soundness of parameter $\lambda$.

\begin{lemma}
	For all $\lambda_2 > \lambda_1 > 0$, if $\Rule_1 \gets \SLIC(\mathbf{X},\mathbf{y},k,\lambda_1)$ and $\Rule_2 \gets \SLIC(\mathbf{X},\mathbf{y},k,\lambda_2)$, then $|{\Rule_1}| \leq |{\Rule_2}|$ and $\cE_{\Rule_1} \geq \cE_{\Rule_2}$.  
\end{lemma}

\begin{proof}
	First, note that construction of $Q^k$ depends only on $\mathbf{X}$ and $\mathbf{y}$. Furthermore, the parameter $\lambda$ influences only the associated weight function. We denote weight functions corresponding to $\lambda_1$ and $\lambda_2$ as $W_{\lambda_1}$ and $W_{\lambda_2}$ respectively. Furthermore, let $\sigma_1 = \MaxSAT(Q^k, W_{\lambda_1})$ and $\sigma_2 = \MaxSAT(Q^k, W_{\lambda_1})$. If $\sigma_1 = \sigma_2$, the lemma trivially holds. We now complete proof by contradiction argument for the case when $\sigma_1 \neq \sigma_2$. 
	
	Let $|{\Rule_1}| > |{\Rule_2}|$. As $\sigma_1 \neq \sigma_2$, we have $W_{\lambda_2} (\sigma_1) \leq W_{\lambda_2} (\sigma_2)$. Since $W_{\lambda} (\sigma) = |{\Rule}| + \lambda \cE_{\Rule}$, where $\Rule$ is extracted from $\sigma$ as stated above. Therefore, we have $\lambda_2(\cE_{\Rule_2} - \cE_{\Rule_1}) \geq |\Rule_1| - |{\Rule_2}|$. But we also have $W_{\lambda_1} (\sigma_1) \leq W_{\lambda_1} (\sigma_2)$, which implies that $\lambda_1(\cE_{\Rule_2} - \cE_{\Rule_1}) \leq |{\Rule_1}| - |{\Rule_2}|$. Since $\lambda_1 > \lambda_2$, we have contradiction. Therefore, it must be the case that $|{\Rule_1}| \leq |{\Rule_2}|$. 
	
\end{proof}

\subsection{Illustrate Example}
We illustrate our encoding with the help of a toy example. Let $n = 2, m = 3, k = 2$ and $X = \begin{bmatrix}
1 & 0  & 1\\
0 & 1  & 1\\
\end{bmatrix}$ and $y = \begin{bmatrix}
0 \\ 1
\end{bmatrix}$. Then we have following clauses: \\[1em]
$N_1 := (\neg \eta_1)$; $\qquad \quad$ $N_2 := (\neg \eta_2)$; \\[1em]
$V_{1}^{1} = (\neg b_1^1)$; $\qquad \qquad$ $V_{1}^{2} = (\neg b_1^2)$; $\qquad \qquad$ $V_{1}^{3} = (\neg b_1^3)$;\\[1em]
$V_{2}^{1} = (\neg b_2^1)$; $\qquad \qquad$ $V_{2}^{2} = (\neg b_2^2)$; $\qquad \qquad$ $V_{2}^{3} = (\neg b_2^3)$;  \\[1em]
$D_1 := (\neg \eta_1 \rightarrow (\neg (b_1^1 \vee b_1^3) \vee \neg(b_2^1 \vee b_2^3))$; \\[1em]
$D_2 := (\neg \eta_2 \rightarrow ((b_1^2 \vee b_1^3) \wedge (b_2^2 \vee b_2^3) )$
\subsection{Beyond CNF Rules}
While CNF formulas are general enough to express every Boolean formula, the length of representation may not be polynomial size. Therefore, one might wonder if we can extend {\SLIC} to learn rules in other canonical forms as well. In fact, early CSP based approaches to rule learning focused on rules in DNF form. We now show that with a minor change, we are able to learn rules expressible in DNF.  Suppose that we are interested in learning a rule $S$ that is expressible in DNF, such that $y = S(\mathbf{x})$, where S is a DNF formula. We note that $(\mathbf{y} = S(\mathbf{x})) \leftrightarrow \neg (y = \neg S (\mathbf{x}))$. And if $S$ is a DNF formula, then $\neg S$ is a CNF formula. Therefore, to learn rule $S$, we simply call {\SLIC} with $\neg \mathbf{y}$ as input and negate the learned rule.

\subsection{Complex Objective Functions}\label{sec:generalization}
We now discuss how {\SLIC} can be easily extended to handle complex objective functions. The objective function for {\SLIC} as defined in Equation~\ref{eq:obj-sparsity} treats all features equally. In some cases, the user might prefer rules that contain certain features. Such an extension is fairly easy to achieve as we need only to change the weight function corresponding to clauses $V_i^j$. Furthermore, in certain cases, one might want to minimize the total number of different features across different clauses rather than minimize the total number of terms. Such an extension is fairly easy to handle as we can simply replace  $\bigwedge_{j=1}^{k} V_i^j$ with $\hat{V_i}$ where $\hat{V_i} = (\bigvee_{j=1}^{k} \neg b_{i}^{j})$. It is worth noting that the proposed modifications impact only the {\MaxSAT} query and does not require any modifications to the underlying {\MaxSAT} solver. {\em We believe that such a separation is a key strength of {\SLIC} as it separates modeling and solving completely. }
%


\section{Evaluation}\label{sec:evaluation}
To evaluate the performance of {\SLIC}, we implemented a prototype implementation in Python that employs {\MaxHS}~\cite{DB11} to handle MaxSAT instances. We also experimented with LMHS~\cite{BSJ15}, another state of the art MaxSAT solver and MaxHS outperformed LMHS for our benchmarks~\footnote{A detailed evaluation among different MaxSAT solvers is beyond the scope of this work and left for future work}. We conducted an extensive set of experiments on diverse publicly available benchmarks, seeking to answer the following questions\footnote{The source code of {\SLIC} and benchmarks can be viewed at \protect \url{https://github.com/meelgroup/mlic}}:
\begin{enumerate}
	\item Do advancements in {\MaxSAT} solving enable {\SLIC} to be run with datasets involving tens of thousands of variables with thousands of binary features?
	\item How does the accuracy of {\SLIC} compare to that of state of the art but typically non-interpretable classifiers? 
	\item How does the accuracy of {\SLIC} vary with the size of training set?
	\item How does the accuracy of {\SLIC} vary with $\lambda$?
	\item How does the size of learnt rules of {\SLIC} vary with $\lambda$?
	\end{enumerate}
	
	In summary, our experiments demonstrate that {\SLIC} can handle datasets involving tens of thousands of variables with thousands of binary features. Furthermore, {\SLIC} can generate rules that are not only interpretable but with accuracy comparable to that of other competitive classifiers, which often produce hard to interpret rules/models. 
	We demonstrate that {\SLIC} is able to achieve sufficiently high accuracy with very few samples. 
	
	
	\begin{table*}
		{\scriptsize
			\centering
			\ra{2}
			\begin{tabular}{@{}cccccccccc@{}}\toprule
				\textbf{Dataset} & \textbf{Size} & \textbf{\# Features}& \textbf{RIPPER} & \textbf{Log Reg} & \textbf{NN} & \textbf{RF} & \textbf{SVC} & \textbf{{\SLIC}} 
				\\ \hline
				TomsHardware& 28170 & 830 &
		\shortstack{ 0.968\\(92.8)} & \shortstack{0.976\\(0.2)} & \shortstack{0.977\\(3.4)} & \shortstack{ 0.976\\(64.9 )} & \shortstack{     Timeout}
			 & \shortstack{0.969\\(2000)} 
				\\ \hline
				Twitter& 49990 & 1050 &
				\shortstack{ 0.938\\(187.3)} & \shortstack{0.963\\(0.2)} & \shortstack{0.965\\(6.8)} & \shortstack{0.962\\(250.9 )} & \shortstack{0.962\\(1010.0)} 
				& \shortstack{0.958\\ (2000)} 
				\\ \hline
				adult-data& 32560 & 262 &
				 \shortstack{  0.852\\(0.5)} & \shortstack{0.801\\(0.3)} & \shortstack{0.866\\(3.0)} & \shortstack{ 0.844\\(41.8 )} & \shortstack{       Timeout}
				& \shortstack{0.755\\ (2000)}
				\\ \hline
				credit-card&  30000 & 334&
				\shortstack{  0.811\\(0.7)} & \shortstack{0.781\\(0.1)} & \shortstack{0.822\\(3.9)} & \shortstack{  0.82\\(25.5 )} & \shortstack{       Timeout }
				& \shortstack{0.82\\ (2000)} 
				\\ \hline
				ionosphere&350&564&
				\shortstack{  0.886\\(0.1)} & \shortstack{0.909\\(0.1)} & \shortstack{0.926\\(1.2)} & \shortstack{  0.909\\(1.3 )} & \shortstack{  0.886\\(0.1 )}
				& \shortstack{0.889\\ (15.04)}
				\\ \hline
				PIMA& 760&134&
				\shortstack{ 0.774\\(0.1)} & \shortstack{0.749\\(0.1)} & \shortstack{0.764\\(1.3)} & \shortstack{   0.761\\(1.3)} & \shortstack{   0.77\\(21.4 )}
				& \shortstack{0.736\\ (2000)} 
				\\ \hline
				parkinsons&190&392&
				\shortstack{ 0.868\\(0.1)} & \shortstack{0.884\\(0.1)} & \shortstack{0.921\\(1.2)} & \shortstack{   0.895\\(1.1)} & \shortstack{   0.879\\(1.6 )}
				&\shortstack{0.895\\ (245)} 
				\\ \hline
				Trans& 740&64&
				\shortstack{  0.78\\(0.0)} & \shortstack{0.759\\(0.0)} & \shortstack{0.788\\(1.2)} & \shortstack{  0.788\\(1.2 )} & \shortstack{0.765\\(372.3 )}
				& \shortstack{0.797\\ (1177)} 
				\\ \hline
				WDBC &560&540&
			\shortstack{  0.961\\(0.1)} & \shortstack{0.936\\(0.0)} & \shortstack{0.961\\(1.3)} & \shortstack{  0.943\\(1.4 )} & \shortstack{  0.955\\(3.0 )}
				& \shortstack{0.946\\ (911)} 
				\\ \hline
				\bottomrule
				\end{tabular}
				\caption{Comparison of classification accuracy with 10-fold cross validation for different classifiers. For every cell in the last five columns, the top value represents the accuracy, while the value sorrounded by parenthesis represent average training time.
					}
					\label{tab:performance}
					}
					\end{table*}
					
			\begin{table*}
				{\scriptsize
					\centering
					\begin{tabular}{@{}ccccc@{}}\toprule
						\textbf{Dataset} & \textbf{Size} & \textbf{\# Features}& \textbf{RIPPER} & \textbf{{\SLIC}} 
						\\ \hline
						TomsHardware& 28170 & 830 &
						57.5 & 
						 4
						\\ \hline
						Twitter& 49990 & 1050 &
						78.5 &
						15
						\\ \hline
						adult-data& 32560 & 262 &
						74.5 &
						51.5
						\\ \hline
						credit-card&  30000 & 334&
						7.5 &
						4
						\\ \hline
						ionosphere&350&564&
						3 &
						5.5
						\\ \hline
						PIMA& 760&134&
						5  & 9
							
						\\ \hline
						parkinsons&190&392&
						6.5& 
						6
						\\ \hline
						Trans& 740&64&
						6 & 
						4 
						\\ \hline
						WDBC &560&540&
						7.5 &
						3.5
						\\ \hline
						\bottomrule
					\end{tabular}
					\caption{Comparison of RIPPER vis-a-vis {\SLIC} in terms of the size of Rules.
						{\bfseries Note that despite using only a small number of literals, the proposed classifier, {\SLIC} mostly has better accuracy than RIPPER.}}
					\label{tab:ruleSize}
				}
			\end{table*}			
					
	\subsection{Experimental Methodology}
	We conducted extensive experiments on publicly available data sets obtained from UCI repository~\cite{blake1998uci}. The data sets involved both real- and categorical-valued features. Specifically, the specific datasets are:  buzz events from two different social networks: Twitter, Tom's Hardware,  Adult Data (adult\_data), Credit Approval Data Set (credit\_data), Ionosphere (Ionos), Pima Indians Diabetes (PIMA),  Parkinsons,  connectionist  bench  sonar  (Sonar),
	blood  transfusion  service  center  (Trans),  and  breast
	cancer Wisconsin diagnostic (WDBC).

	For purposes of comparison of the accuracy of {\SLIC}, we considered a variety of  popular classifiers: $\ell_1$-penalized Logistic regression (LogReg), Nearest neighbors
	classifier (NN), and the black box random forests (RF), and support vector classification (SVC).
	

	We perform 10-fold cross-validation to perform an assessment of accuracy on a validation set.  We compute the mean  across the 10 folds for each choice of a regularization (or complexity control) parameter for each technique (baseline and MLIC), and report the best cross-validation accuracy.    The number of parameter values is comparable (~10) for each technique.   For RF and RIPPER we use control based on the cutoff of the number of examples in the leaf node. For SVC and LogReg we discretize the regularization parameter on a logarithmic grid.   In case of {\SLIC} we have 2 choices of $\lambda \in  \{1,10\}$ and number of clauses, $k \in \{ 1,2,3\}$ and the type of rule as \{CNF, DNF\}.   We set the training time cutoff for each classifier (on each fold) to be 2000 seconds.   {\em Again, note that some classifiers can be much faster than others, but in this paper we focus on the best tradeoff of accuracy vs interpretability in mission-critical settings, and the training time (which can be off-line) is secondary, as long as it is realistic. In this context, note that testing time for each of these techniques is less than 0.01 seconds for a given set of labels.}

	\subsection{Illustrative Example}
	We illustrate the interpertable rules that are computed by {\SLIC} on the iris data set, which is a simple benchmark and widely used by machine learning community to illustrate new classification techniques. We consider the binary problem of
	classifying  iris  versicolor  from  the  other  two  species,
	setosa  and virginica. Of the four features, sepal length,
	sepal  width,  petal  length,  and  petal  width, we learn the following rule:
	$\Rule$:= 
	\begin{enumerate}
		\item (sepal length $>$ 6.3 $\vee$ sepal width $>$ 3.0 $\vee$ petal width $<=$ 1.5 ) $\wedge$ 
		\item ( sepal width $<=$ 2.7 $\vee$ petal length $>$ 4.0 $\vee$ petal width $>$ 1.2 ) $\wedge$
		\item ( petal length $<=$ 5.0) 
	\end{enumerate}
	
	Let us pause a bit to understand how to apply the above rule. The above rule implies that when the three constraints are satisfied, the flower must be classified as Iris otherwise, non-iris. The size of the above rule, i.e. $|\Rule| = \Sigma_{i} |C_i| = 3+3+1 = 7$.

	\subsection{Results}
	
	\begin{figure}
		
		\centering
		\includegraphics[scale=0.75]{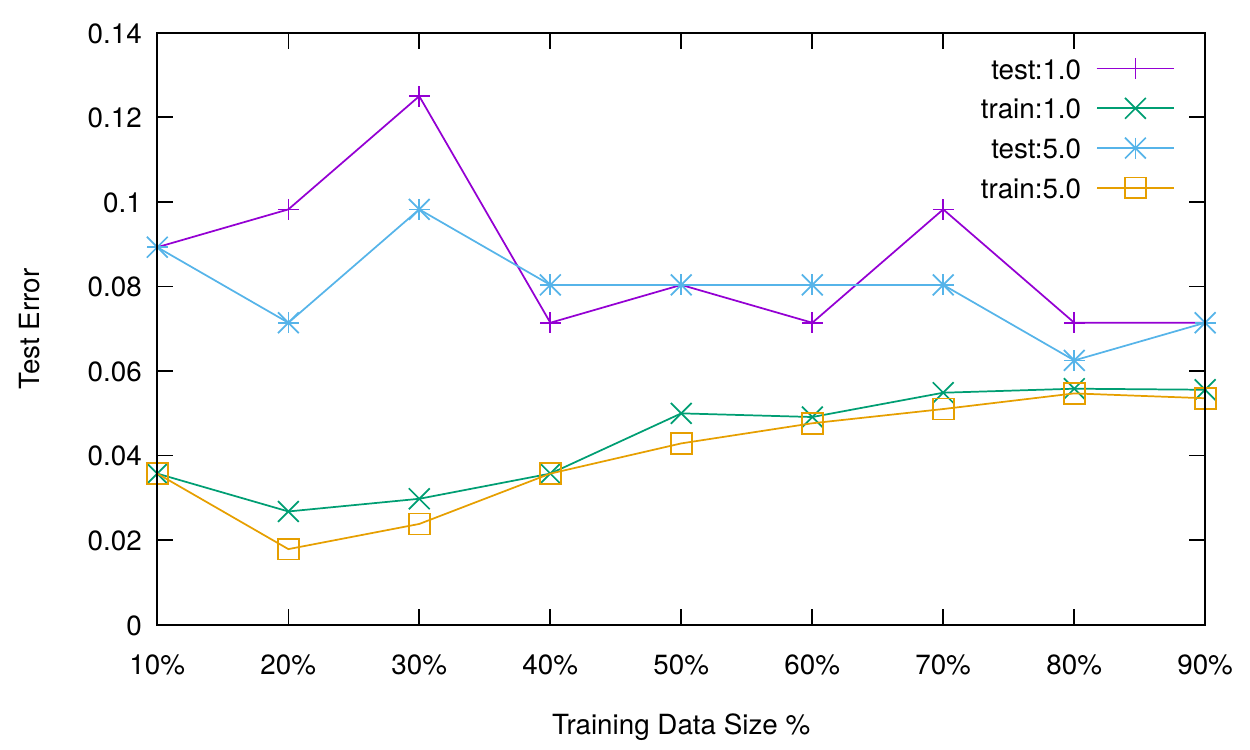}
		\caption{Plot demonstrating behavior of training and test accuracy vs Size of Training data for WDBC.}
		\label{fig:quality}
	\end{figure}

	Table~\ref{tab:performance} presents results of comparison of {\SLIC} vis-a-vis typical non-interpretable classifiers. The first three columns list the name, size (number of samples) and the number of binary features for each Dataset. The next five columns present test accuracy of the classifiers RIPPER, Logistic Regression (Log Reg), Nearest Neighbor (NN), Random Forest (RF), and SVC. The final column contain the median test accuracy for {\SLIC}. For every cell in the last five columns, the top value represents the accuracy, while the value sorrounded by parenthesis represent average training time. We draw the following two conclusions from the table: First, {\SLIC} is able to handle datasets with tens of thousands of examples with hundreds of features. The scalability of {\SLIC} demonstrates the potential presented by remarkable progress in SAT solving. Recent research efforts have often used {\NP}-hardness of the problem to justify the usage of heuristics but our experience with {\SLIC} shows that SAT solving is able to solve many large-scale problems directly. Note that when {\MaxHS} times out, it is able to provide the best solution found so far. In this context, it is worth noting that for some of the benchmarks, even state of the art classifiers such as SVC time out. Secondly, {\SLIC} is often able to achieve accuracy that is sufficiently close to accuracy achieved by 
	typical non-interpretable classifiers but produces easy to state rules that often have just a few literals.

	To demonstrate {\SLIC}'s ability to compute easy to state rules in comparison to the state of the art classifiers such as RIPPER, we computed the size of rules returned by RIPPER and {\SLIC}. Table~\ref{tab:ruleSize} presents results of comparison of {\SLIC} vis-a-vis RIPPER. The first three columns list the name, size (number of samples) and the number of binary features for each Dataset. The next two columns state the median size of rules returned by RIPPER and {\SLIC}. The size of a rule is computed as the number of terms involved in a rule. First, note that except for two cases where RIPPER has produced marginally shorter rules compared to {\SLIC}, {\SLIC} produces significantly shorter rules and sometimes, these rules could be orders of magnitude larger than those produced by {\SLIC}. For example,
	for Toms hardware, the rule produced by RIPPER has 57 terms compared to just 4 literals for {\SLIC}. Note that with {\SLIC} has better accuracy than RIPPER.  
	One might wonder if the rule learned by RIPPER could have been simply transformed into a sparser rule; it is not the case here. Furthermore, it is worth noting that RIPPER does not provide sound
	handle to tune rule size and therefore, user is left to trying out combination of input
	parameters without any guarantee of improvement of the interpretability of generated rules, which we experienced in this case. A in-depth study into failure of RIPPER to generate sparser rules than {\SLIC} is beyond the scope of this work.

	To measure the accuracy of {\SLIC} w.r.t. the size of training data, we consider test errors when only a fraction of training data is available (we vary it from $10$ \% to $90$ \% in steps of $10$ \%). .
	Due to lack of space, we present result for only one benchmark, WDBC, for $\lambda = 1 $ and 5 and $k=1$ in Figure~\ref{fig:quality}. We plot median training and test accuracy of {\SLIC} over 10 trials, which is also known as learning curve in machine learning literature. The y-axis represents the error as the ratio of incorrect predictions to total examples  while the x-axis represents the size of training set. The plot shows how training and test error vary for $\lambda = 1$ and 5. Note that {\SLIC} is able to achieve sufficiently high test accuracy with just 40\% of the 
	complete dataset.  We observe similar behavior for other benchmarks as well.  
	
	\begin{figure}
		\centering
		
		\includegraphics[scale=0.65]{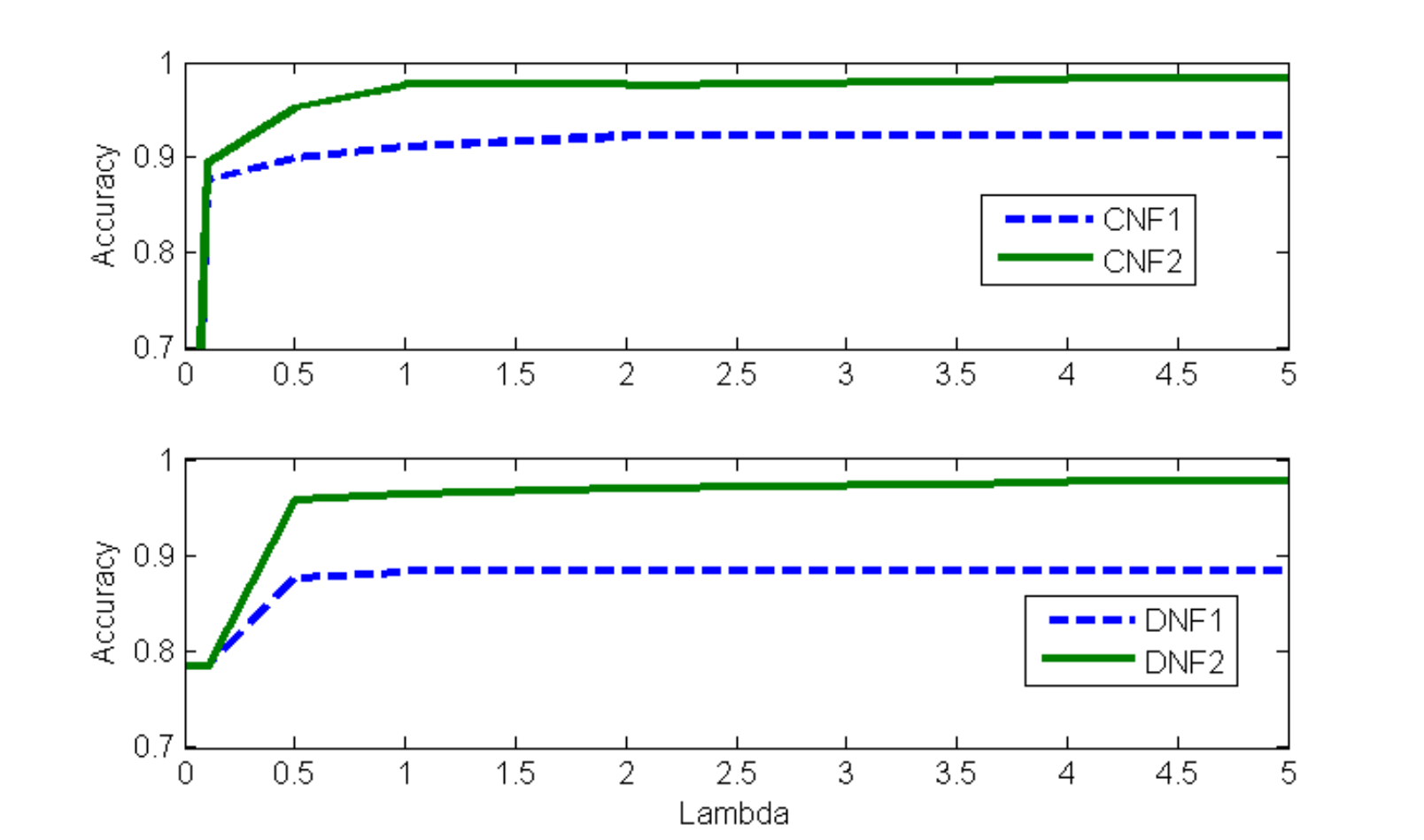}
		\caption{Plot demonstrating monotone behavior of training accuracy vs $\lambda$ for CNF and DNF rules with $k = 1$ and $2$.}
		\label{fig:acc-lambda}
	\end{figure}

	Figures~\ref{fig:acc-lambda} and~\ref{fig:rule-lambda} 
	illustrate how training accuracy and rule sizes vary with $\lambda$ for one of the representative benchmark, parkinsons. 
	CNF1, CNF2, DNF1, DNF2 refer to invocations of {\SLIC} with (rule type, k) set to (CNF, 1), (CNF, 2), (DNF,1), and (DNF,2) 
	respectively. For each of the plots, x-axis refers to the value of $\lambda$ while y-axis represents Rule size (i.e. $|\Rule|$)  and accuracy for  Figure~\ref{fig:rule-lambda}  and Figure~\ref{fig:acc-lambda} respectively.  First, note that for both CNF and DNF, the accuracy of rules is generally 
	higher for larger k. Significantly, the plots clearly demonstrate monotonicity of rule size and accuracy with respect to $\lambda$.  In contrast, the state of the art interpretable classifier, RIPPER, can lead to rules that can be order of magnitude larger than those produced by {\SLIC}. For example, for Toms 
	hardware, the rule produced by RIPPER has 57 terms compared to just 4 literals for {\SLIC}. In this context, it is worth noting that RIPPER does not provide sound handle to tune rule size and therefore, user is left to trying out combination of input parameters without any guarantee of improvement of the interpretability of generated rules. 
		\begin{figure}
		\centering
		\includegraphics[scale=0.65]{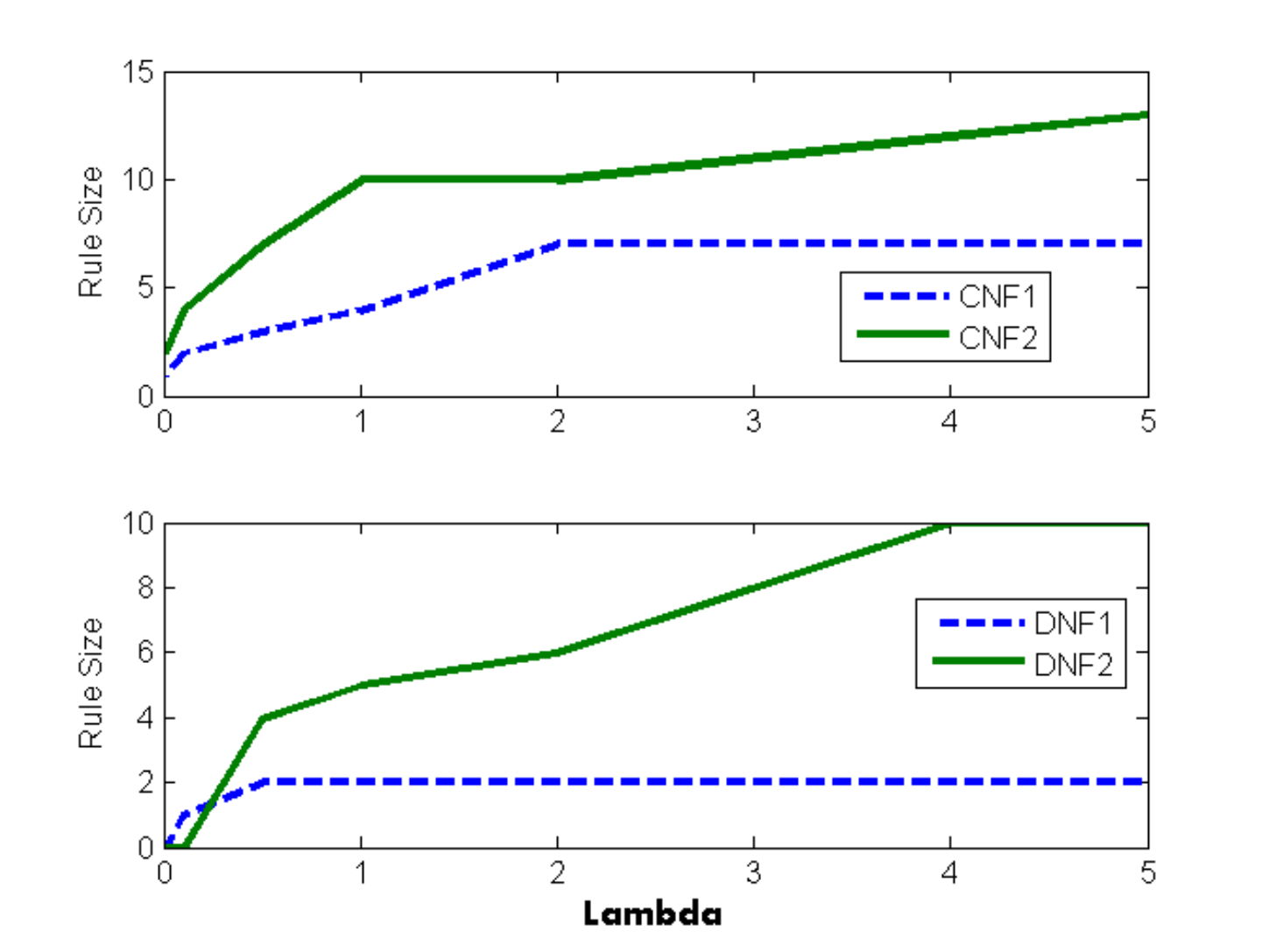}
		\caption{Plot demonstrating behavior of rule size vs $\lambda$}
		\label{fig:rule-lambda}
	\end{figure}
				

\section{Related Work}\label{sec:relatedWork}

There is a long history of learning interpretable classification models from data, including 
popular approaches such as decision trees \cite{BHO09,quinlan2014}, decision lists \cite{Rivest1987}, 
and classification rules \cite{Cohen1995}. While the form of such classifiers is highly amenable
to human interpretation, unfortunately, most of the objective functions that arise
for these problems are intractable combinatorial optimization problems. Hence, 
most popular existing approaches rely on various greedy heuristics, pruning, and ad-hoc local criteria such as maximizing information gain, coverage, e.t.c. 
For example vaious popular decision rule approaches, such as C4.5.rules \cite{quinlan2014}, 
CN2 \cite{ClarkN1989}, RIPPER \cite{Cohen1995}, SLIPPER \cite{CohenS1999}, all make different 
trade-offs in how they use these heuristic criteria for growing and pruning the rules. 

Recent advances in large-scale optimization and scalable Bayesian inference gave rise to state-of-the-art black box models. However, many of the same advances can also be used in the context of interpretable machine learning models. Some of such
recent proposals include Bayesian approaches \cite{LethamRMM2012,wang2015bayesian}, constraint programming~\cite{VH15}, integer
programming approaches to learn decision trees \cite{BertsimasCR2012}, quadratic programming
relaxation with a variance-penalized margin objective \cite{RuckertK2008}. Greedy 
approaches are used with a principled objective function in ENDER \cite{dembczynski2010ender}
and Set covering machines \cite{marchand2002SetCoverMachine}. \cite{jawanpuria2011KernelRules}  propose a hierarchical kernel learning approach and \cite{friedman2008predictive} use optimization
to combine basic Boolean clauses obtained from decision trees.  Linear Programming relaxations
based on Boolean Compressed Sensing formulation have been used to learn sparse interpretable 
rules and checklists\footnote{Note, however, that the objective functions for the integer program and the LP relaxation in these papers are not the same as sparsity-penalized cost-sensitive 
	classification error. } in \cite{malioutov2013exact,emad2015semiquantitative}. Prior work has considered applications of constraint programming to learning Bayesian networks \cite{VH15} and itemset mining \cite{deRaedt2008constraint,NGD09}. In contrast, we focus on learning sparse interpretable classification rules allowing control of accuracy vs. interpretability.

\section{Extensions}

In the paper, we have focused on decision rules in the DNF or CNF form, which is among the most interpretable
classification methods available.  We now describe a few related classification formulations, which are also amenable to being learned from data using a SAT-based framework. A simple AND-clause can be considered
as a requirement that all of the N literals in the clause are satisfied, while a simple OR-clause requires
that at least 1 of the N literals are satisfied. A useful generalization is a ``K-of-N'' clause \cite{CS96}, 
which is true when at least K of the N literals are satisfied. In particular, it leads to a very popular decision rubric called {\em checklists} or {\em scorecards}, widely used in medicine and finance, where a questionnaire asks some questions (e.g., risk factors), and the total number of positive answers is compared to a pre-determined threshold. LP relaxations have been considered for learning scorecards from data \cite{EVM15}, 
and our {\MaxSAT}-based framework can be directly extended. In the case of multi-class classification, a decision rule may be ambiguous, as it does not specify what multi-class label to use when several contradictory clauses pointing to different
labels are satisfied simultaneously. Decision lists \cite{Rivest1987} enforces an order of evaluation of the rules, resolving this ambiguity. Bayesian frameworks for learning decision lists have been considered recently \cite{LethamRMM2012}. 
Perhaps the most well known interpretable classification scheme is a decision tree, where literals are arranged as nodes in a binary tree, and a decision is made by following the path from the root node to one of the leafs. The decision tree can be converted to an equivalent set of classification rules which correspond to all the paths from the root to the leafs, a more expensive representation. On the other side, however, certain small decision rules can lead to very complex decision trees, for example, the "K-of-N" rule cannot be efficiently encoded using a decision tree. 
Recent work has considered combinatorial optimization to learn compact interpretable decision trees \cite{BertsimasCR2012}. 
Beyond simple Boolean expressions, a variety of weighted classification methods can be used, for example, a weighted
linear combination of simple AND clauses -- for instance by using Boosting on a set of classifiers based on simple logical clauses.  In future work, we plan to extend our {\MaxSAT}-based framework for all these related interpretable classification approaches.


\section{Conclusion}\label{sec:conclusion}

We proposed a new approach to learn interpretable classification rules via reduction to ({\MaxSAT}). Due to the impressive advances in {\MaxSAT}-solving, our formulation can find optimal or near-optimal rules
balancing accuracy and interpretability (sparsity) for large data-sets involving
tens or hundreds of thousands of data points, and hundreds or thousands of features. 
Furthermore, the approach separates the modeling from the optimization, and
this framework could be used to solve a wide variety of interpretable classification
formulations, including decision lists, decision trees, and decision rules with 
different cost functions (including group-sparsity, sharing of the variables, and having 
prior knowledge on variable importance). Finally, we demonstrate on experiments that
for many classification problems interpretability does not have to come at a high 
cost in terms of accuracy. 

Furthermore, we hope to share our excitement with applications of constraint programming/{\MaxSAT} in Machine Learning, and to encourage researchers in both interpretable classification and in the CSP/SAT communities to consider this topic further: both in developing new SAT-based formulations for interpretable ML, and in designing bespoke solvers attuned to the problem of interpretable ML. 

\paragraph{Acknowledgements}
This work was supported in part by NUS ODPRT Grant, R-252-000-685-133 and IBM PhD Fellowship. The computational work for this article was  performed on resources of the National Supercomputing Centre, Singapore \url{https://www.nscc.sg}

	{

	}
\end{document}